\documentclass{article} 
\usepackage{hyperref}
\usepackage{url}
\usepackage[utf8]{inputenc}
\usepackage{amsmath}
\usepackage{amssymb}
\usepackage{relsize}
\usepackage{bbm}
\usepackage{amsthm}
\usepackage{graphicx}
\usepackage{caption}
\usepackage{subcaption}

\title{Gamma Processes, Stick-Breaking, and Variational Inference}

\author{
Anirban Roychowdhury, Brian Kulis \\
Department of Computer Science and Engineering\\
The Ohio State University\\
\texttt{roychowdhury.7@osu.edu},kulis@cse.ohio-state.edu
}

\newtheorem{thm}{Theorem}

\begin{document}

\maketitle

\begin{abstract}

While most Bayesian nonparametric models in machine learning have focused on the Dirichlet process, the beta process, or their variants, the gamma process has recently emerged as a useful nonparametric prior in its own right.  Current inference schemes for models involving the gamma process are restricted to MCMC-based methods, which limits their scalability.
In this paper, we present a variational inference framework for models involving gamma process priors. Our approach is based on a novel stick-breaking constructive definition of the gamma process. We prove correctness of this stick-breaking process by using the characterization of the gamma process as a completely random measure (CRM), and we explicitly derive the rate measure of our construction using Poisson process machinery.  We also derive error bounds on the truncation of the infinite process required for variational inference, similar to the truncation analyses for other nonparametric models based on the Dirichlet and beta processes.  Our representation is then used to derive a variational inference algorithm for a particular Bayesian nonparametric latent structure formulation known as the infinite Gamma-Poisson model, where the latent variables are drawn from a gamma process prior with Poisson likelihoods.
Finally, we present results for our algorithms on nonnegative matrix factorization tasks on document corpora, and show that we compare favorably to both sampling-based techniques and variational approaches based on beta-Bernoulli priors.  
\end{abstract}

\section{Introduction}
The gamma process is a versatile pure-jump L\'{e}vy process with widespread applications in various fields of science. Of late it is emerging as an increasingly popular prior in the Bayesian nonparametric literature within the machine learning community; it has recently been applied to exchangeable models of sparse graphs~\cite{graphs_gp} as well as for nonparametric ranking models~\cite{pl_luce_gp}. 
It also has been used as a prior for infinite-dimensional latent indicator matrices~\cite{Gampois}. This latter application is one of the earliest Bayesian nonparametric approaches to count modeling, and as such can be thought of as an extension of the venerable Indian Buffet Process to modeling latent structures where each feature can occur multiple times for a datapoint, instead of being simply binary. 

The flexibility of gamma process models allows them to be applied in a wide variety of Bayesian nonparametric settings, but their relative complexity makes principled inference nontrivial. In particular, all direct applications of the gamma process in the Bayesian nonparametric literature use Markov chain Monte Carlo samplers (typically Gibbs sampling) for posterior inference, which often suffers from poor scalability.
For other Bayesian nonparametric models---in particular those involving the Dirichlet process or beta process---a successful thread of research has considered variational alternatives to standard sampling methods~\cite{v_dp,cv_hdp,ov_hdp}.  One first derives an explicit construction of the underlying ``weights" of the atomic measure component of the random measures underlying the infinite priors; so-called ``stick-breaking" processes for the Dirichlet and beta processes yield such a construction.  Then these weights are truncated and integrated into a mean-field variational inference algorithm.
For instance, stick-breaking was derived for the Dirichlet process in the seminal paper by Sethuraman~\cite{stick}, which was in turn used for variational inference in Dirichlet process models \cite{v_dp}. Similar stick-breaking representations for a special case of the Indian Buffet Process~\cite{ibp_stick} and the beta process~\cite{beta_st} have been constructed, and have naturally led to mean-field variational inference algorithms for nonparametric models involving these priors~\cite{ibp_vi_fdv,beta_st_vi}.  Such variational inference algorithms have been shown to be more scalable than the sampling-based inference techniques normally used; moreover they work with the full model posterior without marginalizing out any variables.

In this paper we propose a variational inference framework for gamma process priors using a novel stick-breaking construction of the process. We use the characterization of the gamma process as a \textit{completely random measure} (CRM), which allows us to leverage Poisson process properties to arrive at a simple derivation of the rate measure of our stick-breaking construction, and show that it is indeed equal to the L\'{e}vy measure of the gamma process.  We also use the Poisson process formulation to derive a bound on the error of the truncated version compared to the full process, analogous to the bounds derived for the Dirichlet process~\cite{ish_james_2001}, the Indian Buffet Process~\cite{ibp_vi_fdv} and the beta process~\cite{beta_st_vi}. We then, as a particular example, focus on the infinite Gamma-Poisson model of \cite{Gampois} (note that variational inference need not be limited to this model). This model is a prior on infinitely wide latent indicator matrices with non-negative integer-valued entries; each column has an associated parameter independently drawn from a gamma distribution, and the matrix values are independently drawn from Poisson distributions with these parameters as means. We develop a mean-field variational technique using a truncated version of our stick-breaking construction, and a sampling algorithm that uses Monte Carlo integration for parameter marginalization, similar to \cite{beta_st}, as a baseline inference algorithm for comparison. 
Finally we compare the two algorithms on a non-negative matrix factorization task involving the Psychological Review, NIPS, KOS and New York Times document corpora.

\noindent \textbf{Related Work.} To our knowledge there has been no previous exposition of an explicit recursive ``stick-breaking"-like construction of the gamma CRM, and by extension no instance of variational algorithms for such priors. The very general inverse L\'{e}vy measure algorithm of \cite{wolp} requires inversion of the exponential integral, as does the generalized CRM construction technique of \cite{orbanz_w} when applied to the gamma process; since the closed form solution of the inverse of an exponential integral is not known, these techniques do not give us an analytic construction of the weights, and hence cannot be adapted to variational techniques in a straightforward manner.  Other constructive definitions of the gamma process include \cite{thithesis}, who discusses a sampling-based scheme for the weights of a gamma process by sampling from a Poisson process.  Further, the characterization of the Dirichlet process as a normalized gamma process may possibly be utilized for sampling gamma process weights, but to our knowledge no existing methods for variational inference employ these approaches.
As an alternative to gamma process-based models for count modeling, recent research has examined the negative binomial-beta process and its variants~\cite{zhou_1,zhou_2,tab_bnb}; the stick-breaking construction of \cite{beta_st} readily extends to such models since they have beta process priors.  The beta stick-breaking construction has also been used for variational inference in beta-Bernoulli process priors \cite{beta_st_vi}, though they have scalability issues when applied to the count modeling problems addressed in this work, as we show in the experimental section.


\section{Background}
\subsection{Completely random measures}

A completely random measure~\cite{crmorig,crmjord} $\mathbb{G}$ on a space $(\Omega, \mathcal{F})$ is defined as a stochastic process on $\mathcal{F}$ such that for any two disjoint Borel subsets $\mathcal{A}_{1} \text{ and } \mathcal{A}_{2}$ in $\mathcal{F}$, the random variables $\mathbb{G}(\mathcal{A}_{1})\text{ and }\mathbb{G}(\mathcal{A}_{2})$ are independent. The canonical way of constructing a completely random measure $\mathbb{G}$ is to first take a $\sigma$-finite product measure $H\text{ on }\Omega\otimes\mathbb{R}^{+}$, then draw a countable set of points $\{(\omega_{k},p_{k})\}$ from a Poisson process on a Borel $\sigma$-algebra on $\Omega\otimes\mathbb{R}^{+}$ with $H$ as the rate measure. Then the CRM is constructed as $\mathbb{G}=\sum_{k=0}^{\infty}p_{k}\delta_{\omega_{k}}$, where the measure given to a measurable Borel set $B\subset \Omega\text{ is }\mathbb{G}(B) = \sum\limits_{k:\omega_{k}\in B}p_{k}$. In this notation $p_{k}$ are referred to as weights and the $\omega_{k}$ as atoms.

If the rate measure is defined on $\Omega\otimes[0,1]$ as $H(d\omega, dp) = cp^{-1}(1-p)^{c-1}B_{0}(d\omega)dp$, where $B_{0}$ is an arbitrary finite continuous measure on $\Omega$ and $c$ is some constant (or function of $\omega$), then the corresponding CRM constructed as above is known as a beta process. 
If the rate measure is defined as $H(d\omega, dp) = cp^{-1}e^{-cp}G_{0}(d\omega)dp$, with the same restrictions on $c$ and $G_{0}$, then the corresponding CRM constructed as above is known as the gamma process. 
The total mass of the gamma process $G, G(\Omega)$, is distributed as $\text{Gamma}(cG_{0}(\Omega),c)$. The improper distributions in these rate measures integrate to infinity over their respective domains, ensuring a countably infinite set of points in a draw from the Poisson process. For the beta process, the weights $p_{k}$ are in [0,1], whereas for the gamma process they are in $[0,\infty)$. In both cases however the sum of the weights is finite, as can be seen from Campbell's theorem \cite{crmorig},  and is governed by $c$ and the total mass of the base measure on $\Omega$.
For completeness we note that completely random measures as defined in \cite{crmorig}  have three components: a set of fixed atoms, a deterministic measure (usually assumed absent), and a random discrete measure. It is this third component that is explicitly generated using a Poisson process, though the fixed component can be readily incorporated into this construction~\cite{kingman_pp}. 

If we create an atomic measure by normalizing the weights $\{p_{k}\}$ from the gamma process, i.e. $D=\sum_{k=0}^{\infty}\pi_{k}\delta_{\omega_{k}}$ where $\pi_{k}=p_{k}/\sum_{i=0}^{\infty}p_{i}$, then $D$ is known as a \textit{Dirichlet process} \cite{fergs_dp}, denoted as $D\sim \text{DP}(\alpha_{0},H_{0})$ where $\alpha_{0}=G_{0}(\Omega)\text{ and }H_{0}=G_{0}/\alpha_{0}$. It is not a CRM as the random variables induced on disjoint sets lack independence because of the normalization; it belongs to the class of normalized random measures with independent increments (NRMIs).

\subsection{Stick-breaking for the Dirichlet and Beta Processes}

A recursive way to generate the weights of random measures is given by stick-breaking, where a unit interval is subdivided into fragments based on draws from suitably chosen distributions. For example, the sick-breaking construction of the Dirichlet process~\cite{stick} is given by 
\begin{equation*}
D = \sum\limits_{i=1}^{\infty}V_{i}\prod_{j=1}^{i-1}(1-V_{j})\delta_{\omega_{i}},
\end{equation*}
where $V_{i}\overset{iid}{\sim}\text{Beta}(1,\alpha), \quad \omega_{i}\overset{iid}{\sim}H_{0}$. Here the length of the first break from a unit-length stick is given by $V_{1}$. In the next round, a fraction $V_{2}$ of the remaining stick of length $1-V_{1}$ is broken off, and we are left with a piece of length $(1-V_{2})(1-V_{1})$. The length of the piece in the next round is therefore given by $V_{3}(1-V_{2})(1-V_{1})$, and so on. Note that the weights belong to (0,1), and since this is a normalized measure, the weights sum to 1 almost surely. This is consistent with the use of the Dirichlet process as a prior on probability distributions.

This construction was generalized in \cite{beta_st} to yield stick-breaking for the beta process:
\begin{equation}
\label{beta_st}
B = \sum\limits_{i=1}^{\infty}\sum\limits_{j=1}^{C_{i}}V_{ij}^{(i)}\prod_{l=1}^{i-1}(1-V_{ij}^{(l)})\delta_{\omega_{ij}},
\end{equation}
where $V_{ij}^{(i)}\overset{iid}{\sim}\text{Beta}(1,\alpha),\quad C_{i} \overset{iid}{\sim}\text{Poisson}(\gamma), \quad \omega_{ij}\overset{iid}{\sim}\frac{1}{\gamma}B_{0}$. We use this representation as the basis for our stick breaking-like construction of the Gamma CRM, and use Poisson process-based proof techniques similar to \cite{beta_st_pp} to derive the rate measure.

\section{The Stick-breaking Construction of the Gamma Process}

\subsection{Constructions and proof of correctness} 
We propose a simple recursive construction of the gamma process CRM, based on the stick-breaking construction for the beta process proposed in~\cite{beta_st,beta_st_pp}. In particular, we augment (or `mark') a slightly modified stick-breaking beta process with an independent gamma-distributed random measure and show that the resultant Poisson process has the rate measure $H(d\omega, dp) = cp^{-1}e^{-cp}G_{0}(d\omega)dp$ as defined above. We show this by directly deriving the rate measure of the marked Poisson process using product distribution formulae.
Our proposed stick-breaking construction is as follows:
%
\begin{equation}
\label{gam_st}
G = \sum\limits_{i=1}^{\infty}\sum\limits_{j=1}^{C_{i}}G_{ij}^{(i)}V_{ij}^{(i)}\prod_{l=1}^{i}(1-V_{ij}^{(l)})\delta_{\omega_{ij}},
\end{equation}
where $G_{ij}^{(i)}\overset{iid}{\sim}\text{Gamma}(\alpha+1,c),\quad V_{ij}^{(i)}\overset{iid}{\sim}\text{Beta}(1,\alpha),\quad C_{i} \overset{iid}{\sim}\text{Poisson}(\gamma), \quad \omega_{ij}\overset{iid}{\sim}\frac{1}{\gamma}H_{0}$.  As with the beta process stick-breaking construction, the product of beta random variables allows us to interpret each $j$ as corresponding to a stick that is being broken into an infinite number of pieces.
Note that the expected weight on an atom in round $i$ is $\alpha^{i}/c(1+\alpha)^i$. The parameter $c$ can therefore be used to control the weight decay cadence along with $\alpha$.

The above representation provides the clearest view of the construction, but is somewhat cumbersome to deal with in practice, mostly due to the introduction of the additional gamma random variable. We reduce the number of random variables by noting that the product of a $\text{Beta}(1,\alpha)\text{ and a }\text{Gamma}(\alpha+1,c)$ random variable has an $\text{Exp}(c)$ distribution; we also perform a change of variables on the product of the $(1-V_{ij})$s to arrive at the following equivalent construction, for which we now prove its correctness:
%
\begin{thm}
A gamma CRM with positive concentration parameters $\alpha\text{ and }c$ and finite base measure $H_{0}$ may be constructed as 
\begin{equation}
\label{gam_st2}
G = \sum\limits_{i=1}^{\infty}\sum\limits_{j=1}^{C_{i}}E_{ij}e^{-T_{ij}}\delta_{\omega_{ij}}
\end{equation} where $E_{ij}\overset{iid}{\sim}\text{Exp}(c),\quad T_{ij}\overset{ind}{\sim}\text{Gamma}(i,\alpha), \quad C_{i} \overset{iid}{\sim}\text{Poisson}(\gamma), \quad \omega_{ij}\overset{iid}{\sim}\frac{1}{\gamma}H_{0}$.
\end{thm}
\begin{proof}
Note that, by construction, in each round $i$ in \eqref{gam_st2}, each set of weighted atoms $\{\left(\omega_{ij},E_{ij}e^{-T_{ij}}\right)\}_{j=1}^{C_{i}}$ forms a Poisson point process since the $C_{i}$ are drawn from a Poisson($\gamma$) distribution. In particular, each of these sets is a \textit{marked} Poisson process \cite{kingman_pp}, where the atoms $\omega_{ij}$ of the Poisson process on $\Omega$ are marked with the random variables $E_{ij}e^{-T_{ij}}$ that have a probability measure on $(0,\infty)$. The superposition theorem of \cite{kingman_pp} tells us that the countable union of Poisson process is itself a Poisson process on the same measure space; therefore  denoting $\mathlarger{G_{i} = \sum\limits_{j=1}^{C_{i}}E_{ij}e^{-T_{ij}}\delta_{\omega_{ij}}}$, we can say $\mathlarger{G=\bigcup_{i=1}^{\infty}G_{i}}$ is a Poisson process on $\Omega\times[0,\infty)$. We show below that the rate measure of this process equals that of the Gamma CRM.

Now, we note that the random variable $E_{ij}e^{-T_{ij}}$ has a probability measure on $[0,\infty)$; denote this by $q_{ij}$. We are going to mark the underlying Poisson process with this measure. The density corresponding to this measure can be readily derived using product distribution formulae. To that end, ignoring indices, if we denote $W = \exp{(-T)}$, then we can derive its distribution by a change of variable.  
 Then, denoting $Q=E\times W\text{ where }E\sim \text{Exp}(c),$ we can use the product distribution formula to write the density of $Q$ as 
 \begin{equation*}
 f_{Q}(q) = \int\limits_0^1\frac{\alpha^{i}}{\Gamma(i)}\left(-\log{w}\right)^{i-1}w^{\alpha-2}ce^{-c\frac{q}{w}}\mathrm{d}w,
 \end{equation*}  where $T\sim\text{Gamma}(i,\alpha)$.
 Formally speaking, this is the Radon-Nikodym density corresponding to the measure $q$, since it is absolutely continuous with respect to the Lebesgue measure on $[0,\infty)$ and $\sigma$-finite by virtue of being a probability measure. Furthermore, these conditions hold for all the measures that we have in our union of marked Poisson processes; this allows us to write the density of the combined measure as 
 \begin{align*}
 f(p) &= \sum\limits_{i=1}^{\infty}\int\limits_0^1\frac{\alpha^{i}}{\Gamma(i)}\left(-\log{w}\right)^{i-1}w^{\alpha-2}ce^{-c\frac{p}{w}}\mathrm{d}w \\
 &= \int\limits_0^1\sum\limits_{i=1}^{\infty}\frac{\alpha^{i}}{\Gamma(i)}\left(-\log{w}\right)^{i-1}w^{\alpha-2}ce^{-c\frac{p}{w}}\mathrm{d}w &\text{ by monotone convergence} \\
 &= \int\limits_0^1 \alpha w^{-2}ce^{-c\frac{p}{w}}\mathrm{d}w \\
 &= \alpha p^{-1}e^{-cp} \\
 &= cp^{-1}e^{-cp}\frac{\alpha}{c}
 \end{align*}.
Note that the measure defined on $\mathcal{B}([0,\infty))$ by the ``improper" gamma distribution $p^{-1}e^{-cp}$ is $\sigma$-finite, in the sense that we can decompose $[0,\infty)$ into the countable union of disjoint intervals $[1/k, 1/(k-1)),\quad k=1,2,\ldots\infty$, each of which has finite measure. In particular, the measure of the interval $[1,\infty)$ is given by the exponential integral.
 
Therefore the rate measure of the process G as constructed here is $G(d\omega, dp) = cp^{-1}e^{-cp}G_{0}(d\omega)dp$ where $G_{0}$ is the same as $H_{0}$ up to the multiplicative constant $\frac{\alpha}{c}$, and therefore satisfies the finiteness assumption imposed on $H_{0}$.
\end{proof}
We use the form specified in the theorem above in our variational inference algorithm since the variational distributions on almost all the parameters and variables in this construction lend themselves to simple closed-form exponential family updates.
As an aside, we note that the random variables $(1-V_{ij})$ have a $\text{Beta}(\alpha,1)$ distribution; therefore if we denote $U_{ij}=1-V_{ij}$ then the construction in \eqref{gam_st} is equivalent to
\begin{equation*}
G = \sum\limits_{i=1}^{\infty}\sum\limits_{j=1}^{C_{i}}E_{ij}^{(i)}\prod_{l=1}^{i}U_{ij}^{(l)}\delta_{\omega_{ij}},
\end{equation*}
where $E_{ij}^{(i)}\overset{iid}{\sim}\text{Exp}(c),\quad U_{ij}^{(i)}\overset{iid}{\sim}\text{Beta}(\alpha,1),\quad C_{i} \overset{iid}{\sim}\text{Poisson}(\gamma), \quad \omega_{ij}\overset{iid}{\sim}\frac{1}{\gamma}H_{0}$. This notation therefore relates our construction to the stick-breaking construction of the Indian Buffet Process \cite{ibp_stick}, where the Bernoulli probabilities $\pi_{k}$ are generated as products of iid $\text{Beta}(\alpha,1)$ random variables : $\pi_{1}=\nu_{1}, \quad \pi_{k}=\prod\limits_{i=1}^{k}\nu_{i}\quad \text{where }\nu_{i}\overset{iid}{\sim}\text{Beta}(\alpha,1)$. In particular, we can view our construction as a generalization of the IBP stick-breaking, where the stick-breaking weights are multiplied with independent Exp($c$) random variables, with the summation over $j$ providing an explicit Poissonization.

\subsection{Truncation analysis} 
\label{sec:trunc}
The variational algorithm requires a truncation level for the number of atoms for tractability. Therefore we need to analyze the closeness between the marginal distributions of the data drawn from the full prior and the truncated prior, with the stick-breaking prior weights integrated out. Our construction leads to a simpler truncation analysis if we truncate the number of rounds  (indexed by $i$ in the outer sum), which automatically truncates the atoms to a finite number.  For this analysis, we will use the stick-breaking gamma process as the base measure of a Poisson likelihood process, which we denote by $PP$; this is precisely the model for which we develop variational inference in the next section.  If we denote the gamma process as $G=\sum_{k=0}^{\infty}g_{k}\delta_{\omega_{k}}$, with $g_{k}$ as the recursively constructed weights, then $PP$ can be written as $PP=\sum_{k=0}^{\infty}p_{k}\delta_{\omega_{k}}\text{ where }p_{k} = \text{Poisson}(g_k)$.
Under this model, we can obtain the following result, which is analogous to error bounds derived for other nonparametric models~\cite{ish_james_2001,ibp_vi_fdv,beta_st_vi} in the literature. 
\begin{thm}
 Let N samples $\textbf{X}=(X_{1},..,X_{N})$ be drawn from $PP(G)$. If $G\sim \Gamma\text{P}(c,G_{0})$, the full
 gamma process, then denote the marginal density of $\textbf{X}\text{ as }\textbf{m}_{\infty}(\textbf{X})$. If $G$ is a gamma process truncated after $R$ rounds, denote the marginal density of $\textbf{X}\text{ as }\textbf{m}_{R}(\textbf{X})$. Then 
 \begin{equation*}
 \frac{1}{4}\int|\textbf{m}_{\infty}(\textbf{X})-\textbf{m}_{R}(\textbf{X})|d\textbf{X} \leq 1-\exp\left\lbrace-N\gamma\frac{\alpha}{c}\left(\frac{\alpha}{1+\alpha}\right)^{R}\right\rbrace.
 \end{equation*}
\end{thm} 
\begin{proof}
The starting intuition is that if we truncate the process after R rounds, then the error in the marginal distribution of the data will depend on the probability of positive indicator values appearing for atoms after the $\text{R}^{th}$ round in the infinite version. Combining this with ideas analogous to those in \cite{ish_james_2000} and \cite{ish_james_2001}, we get the following bound for the difference between the marginal distributions:
 \begin{equation*}
 \frac{1}{4}\int|\textbf{m}_{\infty}(\textbf{X})-\textbf{m}_{R}(\textbf{X})|d\textbf{X} \leq \mathbb{P}\left\lbrace\exists (k,j), k>\sum_{r=1}^{R}C_{r}, 1\leq n\leq N \text{ s.t. }X_n(\omega_{kj})>0 \right\rbrace.
 \end{equation*}
 Since we have a Poisson likelihood on the underlying gamma process, this probability can be written as 
 \begin{equation*}
 \mathbb{P}(\cdot) = 1-\mathbb{E}\left[\mathbb{E}\left\lbrace\left(\prod\limits_{r=R+1}^{\infty}\prod\limits_{j=1}^{C_{r}}e^{-\pi_{rj}}\right)^{N} \Bigg |C_{r}\right\rbrace\right],
 \end{equation*} 
 where $\pi_{rj}=G_{rj}^{(r)}V_{rj}^{(r)}\prod_{l=1}^{r}(1-V_{rj}^{(l)})$. We may then use Jensen's inequality to bound it as follows:
 \begin{align*}
 \mathbb{P}(\cdot) &\leq 1-\exp\left[N\sum\limits_{r=R+1}^{\infty}\mathbb{E}\left\lbrace\sum\limits_{j=1}^{C_{r}}\log(e^{-\pi_{rj}})\right\rbrace\right] \\
 &= 1-\exp\left[N\gamma\frac{1}{c}\sum\limits_{r=R+1}^{\infty}\left(\frac{\alpha}{1+\alpha}\right)^{r}\right] \\
 &= 1-\exp\left\lbrace-N\gamma\frac{\alpha}{c}\left(\frac{\alpha}{1+\alpha}\right)^{R}\right\rbrace.
 \end{align*}
\end{proof}

\section{Variational Inference}
As discussed in Section~\ref{sec:trunc}, we will focus on the infinite Gamma-Poisson model, where a gamma process prior is used in conjunction with a Poisson likelihood function.  When integrating out the weights of the gamma process, this process is known to yield a nonparametric prior for sparse, infinite count matrices~\cite{Gampois}.  We note that our approach should easily be applicable to other models involving gamma process priors.

 \subsection{The Model}
 To effectively perform variational inference, we re-write $G$ as a single sum of weighted atoms, using indicator variables $\{d_{k}\}$ for the rounds in which the atoms occur, similar to \cite{beta_st}:
 \begin{equation}
 \label{gp_eqn}
G = \sum\limits_{k=1}^{\infty}E_{k}e^{-T_{k}}\delta_{\omega_{k}},
\end{equation}
where $E_{k}\overset{iid}{\sim}\text{Exp}(c),\quad T_{k}\overset{ind}{\sim}\text{Gamma}(d_{k},\alpha), \quad \sum\limits_{k=1}^{\infty}\mathbbm{1}_{(d_{k}=r)} \overset{iid}{\sim}\text{Poisson}(\gamma), \quad \omega_{k}\overset{iid}{\sim}\frac{1}{\gamma}H_{0}$. We also place gamma priors on $\alpha, \gamma\text{ and }c: \alpha\sim\text{Gamma}(a_{1},a_{2}), \gamma\sim\text{Gamma}(b_{1},b_{2}),c\sim\text{Gamma}(c_{1},c_{2})$.
Denoting the data, the latent prior variables and the model hyperparameters by $\mathcal{D}, \Pi\text{ and }\Lambda$ respectively, the full likelihood may be written as $P(\mathcal{D}, \Pi|\Lambda)=P(\mathcal{D}, \Pi_{-G}|\Pi_{G},\Lambda)\cdot P(\Pi_{G}|\Lambda)\text{ where }P(\Pi_{G}|\Lambda)=P(\alpha)\cdot P(\gamma)\cdot P(c)\cdot P(\mathbf{d}|\gamma)\cdot \prod\limits_{k=1}^{K}P(E_{k}|c)\cdot P(T_{k}|d_{k},\alpha)\cdot \prod\limits_{n=1}^{N}P(z_{nk}|E_{k},T_{k})$.  We truncate the infinite gamma process to $K$ atoms, and take $N$ to be the total number of datapoints. $\Pi_{-G}$ denotes the set of the latent variables excluding those from the Poisson-Gamma prior; for instance, in factor analysis for topic models, this contains the Dirichlet-distributed factor variables (or topics).

From the Poisson likelihood, we have $z_{nk}|E_{k},T_{k}\sim\text{Poisson}(E_{k}e^{-T_{k}})$, independently for each $n$. The distributions of $T_{k}\text{ and }\mathbf{d}$ involve the indicator functions on the round indicator variables $d_{k}$:
\begin{equation*}
P(T_{k}|d_{k},\alpha)=\frac{\alpha^{\nu_{k}(0)}}{\prod\limits_{r\geq 1}\Gamma(r)^{\mathbbm{1}_{(d_{k}=r)}}}T_{k}^{\nu_{k}(1)}e^{-\alpha T_{k}},
\end{equation*}  where $\nu_{k}(s) = \sum\limits_{r\geq 1}(r-s)\mathbbm{1}_{(d_{k}=r)}$, and
\begin{equation*}
P(\mathbf{d}|\gamma)=\prod\limits_{r=1}^{\infty}\frac{\gamma^{\sum_{k}\mathbbm{1}_{(d_{k}=r)}}}{\left(\sum_{k}\mathbbm{1}_{(d_{k}=r)}\right)!}\cdot \exp\left\lbrace-\gamma\mathbb{I}\left(\sum\limits_{r^{'}=r}^{\infty}\sum\limits_{k=1}^{\infty}\mathbbm{1}_{(d_{k}=r^{'})} > 0\right)\right\rbrace.
\end{equation*}See \cite{beta_st_vi} for a discussions on how to approximate these factors in the variational algorithm.

\subsection{The Variational Prior Distribution}
Mean-field variational inference involves minimizing the KL divergence between the model posterior, and a suitably constructed \textit{variational} distribution which is used as a more tractable alternative to the actual posterior distribution. To that end, we propose a fully-factorized variational distribution on the Poisson-Gamma prior as follows:
\begin{equation*}
Q=q(\alpha)\cdot q(\gamma)\cdot q(c)\cdot\prod\limits_{k=1}^{K}q(E_{k})\cdot q(T_{k})\cdot q(d_{k})\cdot\prod\limits_{n=1}^{N}q(z_{nk}),
\end{equation*} 
where $q(E_{k})\sim \text{Gamma}(\acute{\xi_{k}},\acute{\epsilon_{k}}),\quad q(T_{k})\sim \text{Gamma}(\acute{u_{k}},\acute{\upsilon_{k}}),\quad q(\alpha)\sim \text{Gamma}(\kappa_{1},\kappa_{2}),\quad q(\gamma)\sim \text{Gamma}(\tau_{1},\tau_{2}),\quad q(c)\sim \text{Gamma}(\rho_{1},\rho_{2}),\quad q(z_{nk})\sim\text{Poisson}(\lambda_{nk}),\quad q(d_{k})\sim\text{Mult}(\varphi_{k})$.

Instead of working with the actual KL divergence between the full posterior and the factorized proxy distribution, variational inference maximizes what is canonically known as the \textit{evidence lower bound} (ELBO), a function that is the same as the KL divergence up to a constant. In our case it may be written as $\mathcal{L}=\mathbb{E}_{Q}\log P(\mathcal{D}, \Pi|\Lambda) - \mathbb{E}_{Q}\log Q$. We omit the full representation here for brevity.

\subsection{The Variational Parameter Updates}
Since we are using exponential family variational distributions, we leverage the closed form variational updates for exponential families wherever we can, and perform gradient ascent on the ELBO for the parameters of those distributions which do not have closed form updates. We list the updates on the distributions of the prior below. 
The closed-form updates for the hyperparameters in $q(E_{k}), q(\alpha), q(c)\text{ and }q(\gamma)$ are as follows:
\begin{eqnarray*}
\acute{\xi_{k}}=\sum_{n=1}^{N}\mathbb{E}_{Q}(z_{nk})+1,\quad\acute{\epsilon_{k}}=\mathbb{E}(c)+N\times\mathbb{E}_{Q}\left[e^{-T_{k}}\right], \quad \kappa_{1}=\sum_{k=1}^{K}\sum\limits_{r\geq 1}r\varphi_{k}(r)+a_{1}, \\ 
\kappa_{2}=\sum_{k=1}^{K}\mathbb{E}_{Q}(T_{k})+a_{2}, \quad \rho_{1}=c_{1}+K,\quad\rho_{2}=\sum_{k=1}^{K}\mathbb{E}_{Q}(E_{k})+c_{2},\\
\tau_{1}=b_{1}+K,\quad\tau_{2}=\sum\limits_{r\geq 1}\left\lbrace1-\prod\limits_{k=1}^{K}\sum\limits_{\acute{r}=1}^{r-1}\varphi_{k}(\acute{r})\right\rbrace+b_{2}.
\end{eqnarray*}
The updates for the multinomial probabilities in $q(d_{k})$ are given by:
\begin{equation*}
\begin{split}
\varphi_{k}(r)\propto \exp\{ r\mathbb{E}_{Q}(\log\alpha)-\log\Gamma(r)+(r-1)\mathbb{E}_{Q}(\log T_{k})-\zeta\cdot\sum\limits_{i\neq k}\varphi_{i}(r)\\-\mathbb{E}_{Q}(\gamma)\sum\limits_{j=2}^{r}\prod\limits_{k^{'}\neq k}\sum\limits_{r^{'}=1}^{j-1}\varphi_{k^{'}}(r^{'})\}.
\end{split}
\end{equation*}
In addition to these updates, our variational algorithm requires gradient ascent updates on $q(T_{k})$ and updates on $q(\Pi_{-G})\text{ and }q(z_{nk})$  as follows:

The gradients for the two variational parameters in $q(T_{k})$ are:
\begin{equation*}
\begin{split}
\frac{\partial\mathcal{L}}{\partial\acute{u_{k}}}=\sum\limits_{r\geq 1}(r-1)\varphi_{k}(r)\psi^{'}(\acute{u_{k}})-\frac{\mathbb{E}_{Q}(\alpha)}{\acute{\upsilon_{k}}}-\sum\limits_{n=1}^{N}\mathbb{E}_{Q}(E_{k})\left(\frac{\acute{\upsilon_{k}}}{\acute{\upsilon_{k}}+1}\right)^{\acute{u_{k}}}\cdot\log\left(\frac{\acute{\upsilon_{k}}}{\acute{\upsilon_{k}}+1}\right)\\-\sum\limits_{n=1}^{N}\mathbb{E}_{Q}(z_{nk})\frac{1}{\acute{\upsilon_{k}}}-(\acute{u_{k}}-1)\psi^{'}(\acute{u_{k}}) - 1
\end{split}
\end{equation*}
\begin{equation*}
\begin{split}
\frac{\partial\mathcal{L}}{\partial\acute{\upsilon_{k}}}=-\sum\limits_{r\geq 1}(r-1)\varphi_{k}(r)\frac{1}{\acute{\upsilon_{k}}}+\mathbb{E}_{Q}(\alpha)\frac{\acute{u_{k}}}{\left(\acute{\upsilon_{k}}\right)^{2}}-\sum\limits_{n=1}^{N}\mathbb{E}_{Q}(E_{k})\acute{u_{k}}\frac{\acute{\upsilon_{k}}^{\acute{u_{k}}-1}}{(\acute{\upsilon_{k}}+1)^{\acute{u_{k}}+1}}\\+\sum\limits_{n=1}^{N}\mathbb{E}_{Q}(z_{nk})\frac{\acute{u_{k}}}{\left(\acute{\upsilon_{k}}\right)^{2}}- \frac{1}{\acute{\upsilon_{k}}}.
\end{split}
\end{equation*}

For the topic modeling problems, we model the observed vocabulary-vs-document corpus count matrix $D$ as $D\sim \text{Poi}(\Phi Z)$, where the $V\times K$ matrix $\Phi$ models the factor loadings, and the $K\times N\text{ matrix }Z$ models the actual factor counts in the documents. We put the $K-$truncated Poisson-Gamma prior on $Z$, and put a Dirichlet$(\beta_{1},\ldots,\beta_{V})$ prior on the columns of $\Phi$.

The variational distribution $Q$ consequently gets a Dirichlet$(\Phi|\{\textbf{b}\}_{k})$ distribution multiplied to it, where $\textbf{b}=(b_{1},\ldots,b_{V})$ are the variational Dirichlet hyperparameters. This setup does not immediately lend itself to closed form updates for the $b$-s, so we resort to gradient ascent. The gradient of the ELBO with respect to each variational hyperparameter is
\begin{equation*}
\begin{split}
\frac{\partial\mathcal{L}}{\partial b_{vk}} = - \mathbb{E}_{Q}(z_{nk})\cdot\frac{\sum_{v}b_{vk}- b_{vk}}{\left(\sum_{v}b_{vk}\right)^{2}} + \psi^{'}(b_{vk})\cdot\left(\beta_{v}-b_{vk}+\sum_{n}d_{vn}\right)+\psi^{'}(\sum_{v}b_{vk})\times\\\left(\sum_{v}b_{vk}-V-\beta_{v}-\sum_{n}d_{vn}+1\right).
\end{split}
\end{equation*}
In practice however we found a closed-form update facilitated by a simple lower bound on the ELBO to converge faster. We describe the update here. First note that the part of the ELBO relevant to a potential closed form variational update of $\phi_{vk}$ can be written as
\begin{displaymath}
\mathcal{L}= -\phi_{vk}\cdot\sum_{n}\mathbb{E}_{Q}(z_{nk}) + \sum_{n}d_{vn}\cdot\log\phi_{vk} + \cdots,
\end{displaymath} 
which can then be lower bounded as 
\begin{displaymath}
\mathcal{L} \geq \log\phi_{vk}\cdot\left(-\sum_{n}\mathbb{E}_{Q}(z_{nk}) + \sum_{n}d_{vn}\right)+\cdots.
\end{displaymath} 
This allows us to analytically update $b_{vk}$ as $b_{vk}=-\sum_{n}\mathbb{E}_{Q}(z_{nk}) + \sum_{n}d_{vn} + \beta_{v}$. This frees us from having to choose appropriate corpus-specific initializations and learning rates for the $\Phi$s.

A similar lower bound on the ELBO allows us to update the variational parameters of $q(z_{nk})$ as $\lambda_{nk} = -1 - \sum_{v}d_{vn} + \mathbb{E}_{Q}(\log E_{k}) + \mathbb{E}_{Q}(T_{k})$.

\section{The MCMC Sampler}
As a baseline, we also derive and compare with a standard MCMC sampler for this model.  We use the construction in \eqref{gp_eqn} for sampling from the model. To avoid inferring the latent variables in all the atom weights of the Poisson-Gamma prior, we use Monte Carlo techniques to integrate them out, as in \cite{beta_st}. This affects posterior inference for the indicators $z_{nk}$, the round indicators $\mathbf{d}$ and the hyperparameters $c\text{ and }\alpha$. The posterior distribution for $\gamma$ is closed form, as are those for the likelihood latent variables in $\Phi_{-G}$. 
We re-write the construction of the Poisson-Gamma prior:
 \begin{equation*}
G = \sum\limits_{k=1}^{\infty}E_{k}e^{-T_{k}}\delta_{\omega_{k}},
\end{equation*}
$E_{k}\overset{iid}{\sim}\text{Exp}(c),\quad T_{k}\overset{ind}{\sim}\text{Gamma}(d_{k},\alpha), \quad \sum\limits_{k=1}^{\infty}\mathbbm{1}_{(d_{k}=r)} \overset{iid}{\sim}\text{Pois}(\gamma), \quad \omega_{k}\overset{iid}{\sim}\frac{1}{\gamma}H_{0}$. We put improper priors on $\alpha$ and $c$, and a noninformative Gamma prior on $\gamma$. The indicator counts are given by $Z_{nk}\sim\text{Pois}(g_{k}),\text{ where } g_{k}=E_{k}e^{-T_{k}}$. To avoid sampling the atom weights $E_{k}\text{ and }T_{k}$, we integrate them out using Monte Carlo techniques in the sampling steps for the prior.

\subsection{Sampling the round indicators}
The conditional posterior for the round indicators $\mathbf{d}=\left\lbrace d_{k}\right\rbrace_{k=1}^{K}$ can be written as 
\begin{equation*}
p\left(d_{k}=i|\lbrace d_{l}\rbrace_{l=1}^{k-1},\lbrace Z_{nk}\rbrace_{n=1}^{N},\alpha,c,\gamma\right)\propto p\left(\lbrace Z_{nk}\rbrace_{n=1}^{N}|d_{k}=i,\alpha,c\right)p\left(d_{k}=i|\lbrace d_{l}\rbrace_{l=1}^{k-1}\right).
\end{equation*}

For the first factor, we collapse out the stick-breaking weights and approximate the resulting integral using Monte-Carlo techniques as follows:
\begin{align*}
p\left(\lbrace Z_{nk}\rbrace_{n=1}^{N}|d_{k}=i,\alpha,c\right) &= \int_{[0,\infty]^{i}}\prod\limits_{n=1}^{N}\text{Pois}(Z_{nk}|g_{k})\mathrm{d}G\\
&\approx \frac{1}{S}\sum\limits_{s=1}^{S}\prod\limits_{n=1}^{N}\text{Pois}(Z_{nk}|g_{k}^{(s)}),
\end{align*}
where $g_{k}^{(s)} = E_{k}^{(s)}e^{-T_{k}^{(s)}}\overset{d}{=}V_{k,d_{k}}^{(s)}\prod_{l=1}^{d_{k}}(1-V_{kl}^{(s)})$. Here $S$ is the number of simulated samples from the integral over the stick-breaking weights. We take $S=1000$ in our experiments.

The second factor is the same as \cite{beta_st}:
\begin{equation*}
p(d_{k}=d|\gamma,\lbrace d_{l}\rbrace_{l=1}^{k-1}) = 
\left\{
   \begin{array}{lll}
      0 & \mbox{if } d < d_{k-1} \\
      \frac{1-\sum_{t=1}^{D_{k-1}}\text{Pois}(t|\gamma)}{1-\sum_{t=1}^{D_{k-1}-1}\text{Pois}(t|\gamma)} & \mbox{if } d = d_{k-1} \\
      \left(1-\frac{1-\sum_{t=1}^{D_{k-1}}\text{Pois}(t|\gamma)}{1-\sum_{t=1}^{D_{k-1}-1}\text{Pois}(t|\gamma)}\right)(1-\text{Pois}(0|\gamma))\text{Pois}(0|\gamma)^{h-1} & \mbox{if } d = d_{k-1}+h.
   \end{array}
\right.
\end{equation*}
Here $D_{k}\overset{\Delta}{=}\sum\limits_{j=1}^{k}\mathbb{I}(d_{j}=d_{k})$. Normalizing the product of these two factors over all $i$ is infeasible, so we evaluate this product for increasing $i$ till it drops below $10^{-2}$, and normalize over the gathered values.

\subsection{Sampling the factor variables}
Here we consider the Poisson factor modeling scenario that we use to model vocabulary-document count matrices. Recall that a $V\times N$ count matrix $D$ is modeled as $D=\text{Poi}(\Phi Z)$, where the $V\times K$ matrix $\Phi$ models the factor loadings, and the $K\times N\text{ matrix }Z$ models the actual factor counts in the documents.. We put the Poisson-Gamma prior on $Z$ and symmetric Dirichlet$(\beta_{1},\ldots,\beta_{V})$ priors on the columns of $\Phi$. The sampling steps for $\Phi$ and $Z$ are described next.
\subsubsection{Sampling $\Phi$}
First note that the elements of the count matrix are modeled as $d_{vn}=\text{Poi}\left(\sum_{k=1}^{K}\phi_{vk}z_{kn}\right)$, which can be equivalently written as $d_{vn}=\sum_{k=1}^{K}d_{vkn},\quad d_{vkn}=\text{Poi}(\phi_{vk}z_{kn})$. Standard manipulations then allow us to sample the $d_{vkn}$'s from $\text{Mult}(d_{vn};p_{v1n},\ldots,p_{vKn})$ where $p_{vkn}=\phi_{vk}z_{kn}/\sum_{k}^{K}\phi_{vk}z_{kn}$.

Now we have $\phi_{k}\sim\text{Dirichlet}(\beta_{1},\ldots,\beta_{V})$. Using standard relationships between Poisson and multinomial distributions, we can derive the posterior distribution of the $\phi_{k}$'s as $\text{Dirichlet}(\beta_{1}+d_{1k},\ldots,\beta_{V}+d_{Vk}),\text{ where }d_{vk}=\sum_{n=1}^{N}d_{vkn}$.
\subsubsection{Sampling Z}
In our algorithm we sample each $z_{nk}$ conditioned on all the other variables in the model; therefore the conditional posterior distribution can be written as 
\begin{align*}
p(z_{nk}|D, \Phi,Z_{n,-k},\mathbf{d},\alpha,c,\gamma) &= p(D|Z_{n},\Phi)p(z_{nk}|\mathbf{d},\alpha,c,Z_{n,-k}) \\
&= \prod\limits_{v=1}^{V}\text{Poi}\left(d_{vn}|\sum\limits_{k=1}^{K}\phi_{vk}z_{kn}\right)\frac{p(Z_{n}|\mathbf{d},\alpha,c)}{p(Z_{n,-k}|\mathbf{d},\alpha,c)}.
\end{align*}
The distributions in both the numerator and denominator of the second factor can be sampled from using the Monte Carlo techniques described above, by integrating out the stick-breaking weights.


\subsection{Sampling hyperparameters}
As mentioned above, we put a noninformative Gamma prior on $\gamma$ and improper (1) priors on $\alpha$ and $c$. The posterior sampling steps are described below:
\subsubsection{Sampling $\gamma$}
Given the round indicators $\mathbf{d}=\left\lbrace d_{k}\right\rbrace$, we can recover the round-specific atom counts as described above. Then the conjugacy between the Gamma prior on $\gamma$ and the Poisson distribution of $C_{i}$ gives us a closed form posterior distribution for $\gamma$:
$p(\gamma|\mathbf{d},Z,\alpha,c) = \text{Gamma}(\gamma|a+\sum_{i=1}^{K}C_{i},b+d_{K})$.
\subsubsection{Sampling $\alpha$}
The conditional posterior distribution of $\alpha$ may be written as:
\begin{displaymath}
p(\alpha|Z,\mathbf{d},c)\propto p(\alpha)\prod_{n=1}^{N}\prod_{k=1}^{K}p(Z|\mathbf{d},\alpha,c). 
\end{displaymath}
We calculate the posterior distribution of $Z$ using Monte Carlo techniques as described above. Then we discretize the search space for $\alpha$ around its current values as $\left(\alpha_{cur}+t\Delta\alpha\right)_{t=L}^{U}$, where the lower and upper bounds $L$ and $U$ are chosen so that the unnormalized posterior falls below $10^{-2}$. The search space is also clipped below at 0. $\alpha$ is then drawn from a multinomial distribution on the search values after normalization.
\subsubsection{Sampling c}
We sample $c$ in exactly the same way as $\alpha$. We first write the conditional posterior as 
\begin{displaymath}
p(c|Z,\mathbf{d},\alpha)\propto p(c)\prod_{n=1}^{N}\prod_{k=1}^{K}p(Z|\mathbf{d},\alpha,c). 
\end{displaymath}
The search space $(c>0)$ is then discretized using appropriate upper and lower bounds as above, and $Z$ is sampled using Monte Carlo techniques. $c$ is then drawn from a multinomial distribution on the search values after normalization.

\section{Experiments}
We consider the problem of learning latent topics in document corpora. Given an observed set of counts of vocabulary words in a set of documents, represented by say a $V\times N$ count matrix, where $V$ is the vocabulary size and $N$ the number of documents, we aim to learn $K$ latent factors and their vocabulary realizations using Poisson factor analysis. In particular, we model the observed corpus count matrix $D$ as $D\sim \text{Poi}(\Phi\mathbf{I})$, where the $V\times K$ matrix $\Phi$ models the factor loadings, and the $K\times N\text{ matrix }\mathbf{I}$ models the actual factor counts in the documents. As a baseline, we also derive and compare with a standard MCMC sampler for this model.  We use the construction in \eqref{gp_eqn} for sampling from the model. To avoid inferring the latent variables in all the atom weights of the Poisson-Gamma prior, we use Monte Carlo techniques to integrate them out, as in \cite{beta_st}. This affects posterior inference for the indicators $z_{nk}$, the round indicators $\mathbf{d}$ and the hyperparameters $c\text{ and }\alpha$. The posterior distribution for $\gamma$ is closed form, as are those for the likelihood latent variables in $\Pi_{-G}$. The complete updates are described in the supplementary.

We implemented and analyzed the performance of three variational algorithms corresponding to three different priors on $\mathbf{I}$: the Poisson-gamma process prior from this paper (abbreviated hereafter as VGP), the Bernoulli-beta prior from \cite{beta_st_vi} (VBP) and the IBP prior from \cite{ibp_vi_fdv} (VIBP), along with the MCMC sampler mentioned above (SGP). For the Bernoulli-beta priors we modeled $\mathbf{I}$ as $\mathbf{I}=W\circ Z$ as in \cite{beta_st_vi}, where the nonparametric priors are put on $Z$ and a vague Gamma prior is put on $W$. For the VGP and SGP models we set $\mathbf{I}=Z$. In addition, for all four algorithms, we put a symmetric Dirichlet$(\beta_{1},\ldots,\beta_{V})$ prior on the columns of $\Phi$. We added corresponding variational distributions for the variables in the collection denoted as $\Pi_{-G}$ above. 
We use held-out per-word test log-likelihoods and times required to update all variables in $\Pi$ in each iteration as our comparison metrics, with 80\% of the data used for training. We used the same likelihood metric as \cite{zhou_1}, with the samples replaced by the expectations of the variational distributions.

\textbf{Synthetic Data.} As a warm-up, we consider the performances of VGP and SGP on some synthetic data generated from this model. We generate 200 weighted atoms from the gamma prior using the stick-breaking construction, and use the Poisson likelihood to generate 3000 values for each atom to yield the indicator matrix $Z$. We simulated a vocabulary of 200 terms, generated a 200$\times$200 factor-loading matrix $\Phi$ using symmetric Dirichlet priors, and then generated $D=\text{Poi}(\Phi Z)$. 
For the VGP, we measure the test likelihood after every iteration and average the results across 10 random restarts. These measurements are plotted in fig.\ref{figr:synt}. As shown, VGP's measured heldout likelihood converges within 10 iterations. The SGP traceplot shows the first thirty heldout likelihoods measured after burn-in. Per-iteration times were 15 seconds and 2.36 minutes for VGP (with $K$=125) and SGP respectively. The SGP learned $K$ online, with values oscillating around 50. SNBP refers to the Poisson-Gamma mixture (``NB process") sampler from \cite{zhou_1}. Its traceplot shows the first 30 likelihoods measured after 1000 burn-in iterations. We see that it performed similarly to our algorithms, though slightly worse.

\textbf{Real data.} We used a similar framework to model the count data from the Psychological Review (PsyRev)\footnote{http://psiexp.ss.uci.edu/research/programs\_data/toolbox.htm}, NIPS\footnote{http://www.stats.ox.ac.uk/~teh/data.html}, KOS\footnote{\label{uci}https://archive.ics.uci.edu/ml/datasets/Bag+of+Words} and New York Times\footnotemark[\value{footnote}] corpora. The vocabulary sizes are 2566, 13649, 6906 and 100872 respectively, while the document counts are 1281, 1740, 3430 and 300000 respectively. For each dataset, we ran all three variational algorithms with 10 random restarts each, measuring the held-out log-likelihoods and per-iteration runtimes for different values of the truncation factor $K$. The learning rates for gradient ascent updates were kept on the order of $10^{-4}$ for both VGP and VBP, with 5 gradient steps per iteration. 
A representative subset of results is shown in figs.\ref{figr:psypp} through \ref{figr:nytt}.

We used vague gamma priors on the hyperparameters $\alpha, \gamma\text{ and }c$ in the variational algorithms, and improper (1) priors for the sampler. We found the test likelihoods to be independent of these initializations. The results for the variational algorithms were dependent on the Dirichlet prior $\beta$ on $\Phi$, as noted in fig.\ref{figr:psypp}. We therefore used the learned test likelihood after 100 iterations as a heuristic to select $\beta$. We found the three variational algorithms to attain very similar test likelihoods across all four datasets after a few hours of CPU time, with the VGP and VBP having a slight edge over the VIBP. The sampler somewhat unexpectedly did not attain a competitive score for any dataset, unlike the synthetic case. For instance, as shown in fig.\ref{figr:psypitn}, it oscillated around -7.45 for the PsyRev dataset, whereas the variational algorithms attained -7.23. For comparison, the NB process sampler from \cite{zhou_1} attains -7.25 each iteration after 1000 iterations of burn-in. Also as seen in fig.\ref{figr:psypitn}, VGP was faster to convergence (in less than 10 iterations in $\sim$5 seconds) than VIBP and VBP ($\sim$50 iterations each). The test log-likelihoods after a few hours of runtime were largely independent of the truncation $K$ for the three variational algorithms. Behavior for the other datasets was similar.

Among the three variational algorithms, the VIBP scaled best for small to medium datasets as a function of the truncation factor due to all updates being closed-form, in spite of having to learn the additional weight matrix $W$. The VGP running times were competitive for small values of $K$ for these datasets. However, in the large NYT dataset, VGP was orders of magnitude faster than the Bernoulli-beta algorithms (note the log-scale in fig.\ref{figr:nytt}). For example, with a truncation of 100 atoms, VGP took around 45 seconds per iteration, whereas both VIBP and VBP took more than 3 minutes. The VBP scaled poorly for all datasets, as seen in figs.\ref{figr:psyt} through \ref{figr:nytt}. The reason for this is three-fold: learning the parameters for the additional matrix $W$ which is directly affected by dimensionality (also the reason for VIBP being slow for NYT dataset), gradient updates for two variables (as opposed to one for VGP) and a Taylor approximation required for these gradient updates (see \cite{beta_st_vi}). The sampler SGP required around 7 minutes per iteration for the small datasets and an hour and 40 minutes on average for NYT.

To summarize, we found the VGP to post running times that are competitive with the fastest algorithm (VIBP) in small to medium datasets, and outperform the other methods completely in the large NYT dataset, all the while providing similar accuracy compared to variational algorithms for similar models, as measured by held-out likelihood. It was also the fastest to converge, typically taking less than 15 iterations.  Compared with SGP, our variational method is substantially faster (particularly on large-scale data) and produces higher likelihood scores on real data.

\begin{figure}
\centering
\begin{subfigure}[b]{0.3\textwidth}
\includegraphics[width=\textwidth]{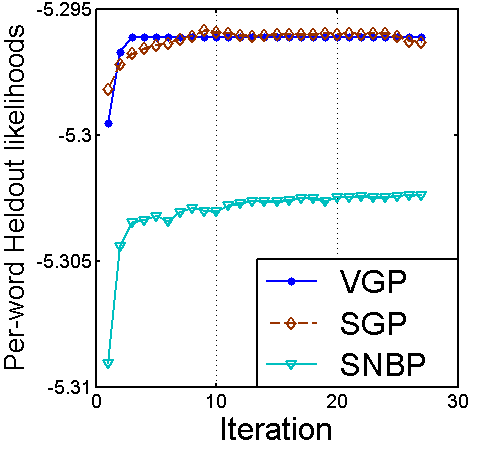}
\caption{ }
\label{figr:synt}
\end{subfigure}
\begin{subfigure}[b]{0.3\textwidth}
\includegraphics[width=\textwidth]{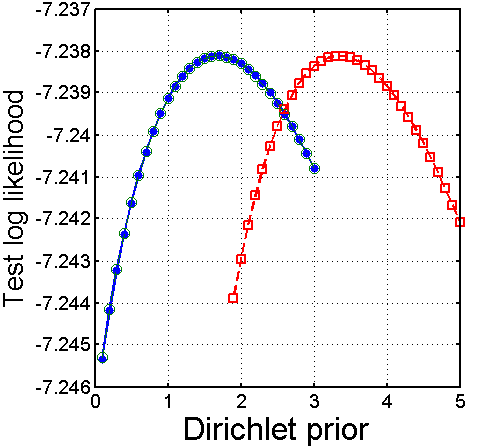}
\caption{ }
\label{figr:psypp}
\end{subfigure}
\begin{subfigure}[b]{0.3\textwidth}
\includegraphics[width=\textwidth]{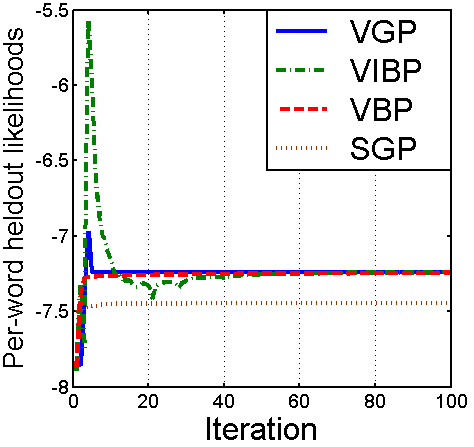}
\caption{ }
\label{figr:psypitn}
\end{subfigure}
\begin{subfigure}[b]{0.3\textwidth}
\includegraphics[width=\textwidth]{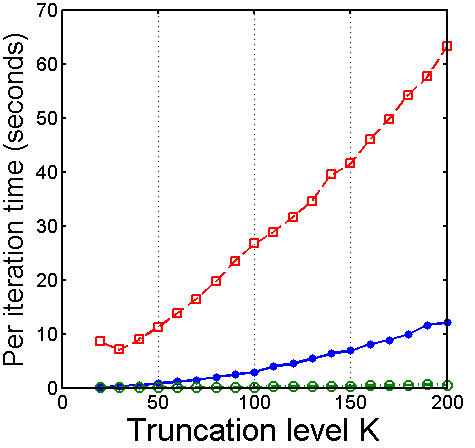}
\caption{ }
\label{figr:psyt}
\end{subfigure}
\begin{subfigure}[b]{0.3\textwidth}
\includegraphics[width=\textwidth]{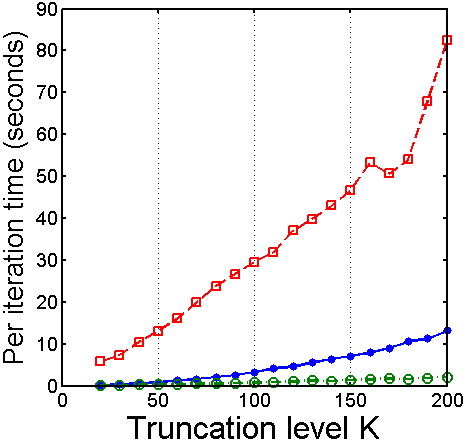}
\caption{ }
\label{figr:nipst}
\end{subfigure}
\begin{subfigure}[b]{0.3\textwidth}
\includegraphics[width=\textwidth]{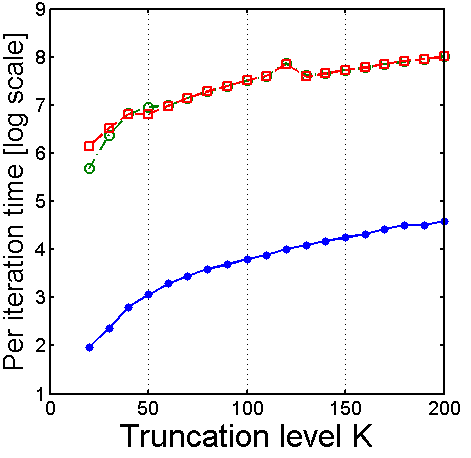}
\caption{ }
\label{figr:nytt}
\end{subfigure}
\caption{Plots of held-out test likelihoods and per-iteration running times. Plots (d), (e) and (f) are for PsyRev, KOS, and NYT respectively. Plots (b) and (c) are for the PsyRev dataset. Algorithm trace colors are common to all plots. See text for full details.}
\label{resfig}
\end{figure}
\section{Conclusion}
We have described a novel stick-breaking representation for gamma processes and used it to derive a variational inference algorithm. This algorithm has been shown to be far more scalable for large datasets than variational algorithms for related models while attaining similar accuracy, and also outperforms sampling-based methods. We expect that recent improvements to variational techniques can also be applied to our algorithm, potentially yielding even further scalability. 
\bibliographystyle{unsrt}
\bibliography{refpaper}

\end{document}